\documentclass[10pt,twocolumn,letterpaper]{article}
\usepackage{iccv}

\usepackage{times}
\usepackage{epsfig}
\usepackage{graphicx}
\usepackage{amsmath}
\usepackage{amssymb}
\usepackage{algorithmic}
\usepackage{algorithm}
\usepackage{dblfloatfix}

\usepackage{amsthm}
\usepackage[export]{adjustbox}

\newtheorem{theorem}{Theorem}[section]
\newtheorem{lemma}{Corollary}[section]
\def\iccvPaperID{8176} 

\usepackage[breaklinks=true,bookmarks=false]{hyperref}

\iccvfinalcopy 

\begin{document}

\title{Multimodal Fusion Refiner Networks}

\author{Sethuraman Sankaran, David Yang, Ser-Nam Lim\\
Facebook AI, New York, NY, USA \\
{\tt\small ssankaran@fb.com, dzyang@fb.com, sernamlim@fb.com}
}

\maketitle

\begin{abstract}
Tasks that rely on multi-modal information typically include a fusion module that combines information from different modalities. In this work, we develop a Refiner Fusion Network (ReFNet) that enables fusion modules to combine strong unimodal representation with strong multimodal representations. ReFNet combines the fusion network with a decoding/defusing module, which imposes a modality-centric responsibility condition. This approach addresses a big gap in existing multimodal fusion frameworks by ensuring that both unimodal and fused representations are strongly encoded in the latent fusion space. We demonstrate that the Refiner Fusion Network can improve upon performance of powerful baseline fusion modules such as multimodal transformers.  The refiner network enables inducing graphical representations of the fused embeddings in the latent space, which we prove under certain conditions and is supported by strong empirical results in the numerical experiments. These graph structures are further strengthened by combining the ReFNet with a Multi-Similarity contrastive loss function. The modular nature of Refiner Fusion Network lends itself to be combined with different fusion architectures easily, and in addition, the refiner step can be applied for pre-training on unlabeled datasets, thus leveraging unsupervised data towards improving performance. We demonstrate the power of Refiner Fusion Networks on three datasets, and further show that they can maintain performance with only a small fraction of labeled data.
\end{abstract}
\section{Introduction}
\begin{figure}[hbt!]
\includegraphics[scale=0.42]{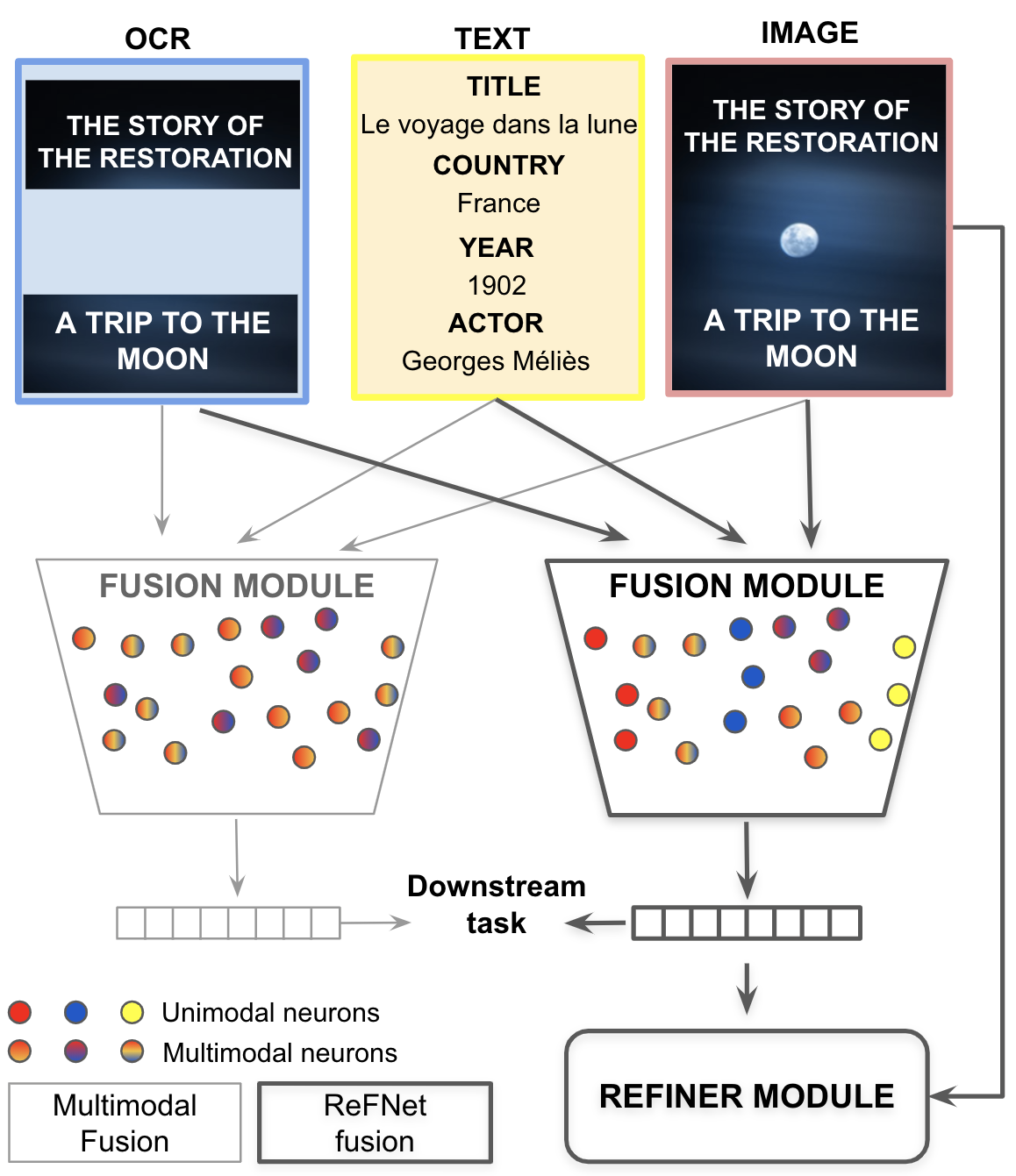}
\caption{Illustration of the idea behind refiner network that infuses a responsibility condition on the fusion architecture.}
\label{fig1}
\end{figure}
In several real-world applications, decision making involves integrating multiple modalities such as vision, text, auditory, and possibly even the content creator and how people engage with the input~\cite{ngiam2011multimodal, atrey2010multimodal, oviatt2003multimodal}. The application of multi-modal inference or decision making systems span several fields such as hate speech detection~\cite{gomez2020exploring}, misinformation detection~\cite{khattar2019mvae}, reasoning tasks~\cite{cadene2019murel},  etc. Multimodal modeling includes two broad steps: extraction of features from different modalities such as images and text, and fusion of the different modalities. Several choices are available for fusion, including late fusion, mid-fusion or early fusion~\cite{gadzicki2020early, minotto2014simultaneous}. Early fusion integrates features extracted from multiple modalities, and uses the integrated feature representation for learning downstream tasks. On the other hand, late fusion integrates classification scores of different features~\cite{poria2017review} to obtain the final classification score. Whether it is early or late fusion, different fusion strategies exist. These include Concat fusion~\cite{wang2020deep}, where different features are concatenated first and input into a MLP to get an integrated feature, as well as Set-based fusion~\cite{reiter2020deep} that is based on permutation invariant functions or Graph-based fusion modules~\cite{angelou2019graph}. \par

The approach in this paper is partly motivated by a limitation of current fusion strategies that predominantly encode multimodal information, ignoring potentially the importance of retaining unimodal signals. Even powerful multimodal architectures such as those with co-attention transformer layers~\cite{lu2019vilbert} across modalities, which attend from the visual stream over language stream, leverage such approaches for the final fusion. \par


We propose a complementary approach wherein we refine these fusion representations to optimally balance independent information that each modality carries with multimodal representations to inform the downstream task. To this end, we leverage notions of responsibility \cite{zhang2019fspool} and contrastive learning~\cite{khosla2020supervised} to build what we refer to as a Refiner Fusion Network (ReFNet). The main idea behind ReFNet is to balance the fusion module with a `responsibility' driven refiner module~\cite{huang2020better, zhang2019deep} as illustrated in Figure~\ref{fig1}. The refiner module drives creation of sets of artificial neurons preferentially tuned to specific modal inputs, while the fusion module drives creation of artificial neurons with mixed representations. The responsibility condition is akin to a regularizer that helps guard against over-fitting of the fusion module to the downstream task, and reduces vanishing gradients with respect to specific input fusion modes. By simultaneously balancing the number of unimodal and multimodal neurons during training, a better representation of the input modalities is achieved to perform multi-modal fusion informed by the downstream task. \par
One interesting finding of this work is that by imposing a responsibility condition in addition to the downstream task, the refiner module induces strong graphical structures between modalities, which we can prove under certain conditions. Such a latent inductive graph, when coupled with metric learning, encourages a latent graphical structure across modalities and samples, which we believe leads to a stronger fusion representation and performance on downstream tasks. More importantly, as opposed to transductive graph modeling ~\cite{liu2009robust}, an inductive graph avoids the burden of needing to ``carry'' an adjacency matrix during inference time, which can get very large in real-world use cases.

Our contributions can be summarized as follow:
\begin{enumerate}
    \item We propose ReFNet, a refiner module that can be added to any given fusion module that helps to induce neurons that are each responsible for a specific modality. We show that ReFNet can boost performance even over powerful transformers and can help induce latent graphical structures that we can show under certain conditions.
    \item When coupled with metric learning, which we call ReFNet\textsubscript{MS}, we observe a further boost in performance, and surmise from the T-SNE plots that we can generate stronger clustering and representation of the different classes. 
    \item Lastly, we demonstrate that ReFNet has increased level of tolerance to lesser amount of labeled data, which also helps reduce annotation needs.
\end{enumerate}

\section{Related Work}

{\bf Responsibility problem} The notion of responsibility was first introduced by Zhang and colleagues~\cite{zhang2019fspool}. The goal of this paper was that when permutation invariant sets are mapped to a latent space, the neurons of the encoding function must be faithful to the discontinuities introduced from the input space to the set entries. This was expanded into Set and Graph Refiner networks in~\cite{huang2020better} wherein an inner loop optimization was performed to divide inputs into set elements that satisfy the responsibility problem, in contrast to say splitting just the featurizer of the CNN of an image. This showed better performance in relational reasoning tasks, and this approach can be used with other permutation invariant architectures such as Deep Multimodal Sets~\cite{reiter2020deep}, \par

{\bf Multitask Learning} The idea behind multitask learning is to learn tasks in parallel but using a shared representation~\cite{caruana1997multitask}. A CentralNet architecture expanded this idea for multimodal fusion networks~\cite{vielzeuf2018centralnet, perez2019mfas, perez2019mfas}. Their approach was to create a central network that links modality specific networks. Each modality is allowed to make decisions independent of other modalities, while a central network aims to leverage the mixed modalities. Taskonomy~\cite{zamir2018taskonomy} builds the shared representation space by first learning several low level tasks that can be of generic value to several downstream tasks. ReFNet is designed to complement these approaches where the refiner operates on the shared representation space to decode back the unimodal signals. \par 


{\bf Graph based Fusion} In Graph based fusion modules, each input modality can be considered nodes of a graph with a known adjacency matrix~\cite{zhang2019graph, angelou2019graph} or explicitly modeling interaction across modalities~\cite{mai2020modality}. The GINFusion model~\cite{xu2018powerful} creates a representation of the graph in an embedding space using a dense graph connection. Adding ReFNet to the downstream loss pushes the fusion architecture to have strong unimodal and strong multimodal components, and to induce edges between modalities, when they exist (as we show later in the manuscript). \par 

{\bf Autoencoder} Autoencoders play a big role in unsupervised learning, transfer learning and dimensionality reduction~\cite{baldi2012autoencoders, burda2015importance, makhzani2015adversarial}. Set autoencoders can be used for dimensionality reduction of a set of features to a lower dimensional space~\cite{chen2018gsae}. Multi-modal autodecoders have been developed for filling in missing data~\cite{jaques2017multimodal}. A special case of the Refiner Network will be the cyclic loss function introduced in~\cite{zhu2017unpaired}. ReFNet enables encoding of feature sets that feed into a fusion module. \par 

{\bf Metric Learning} 
Supervised deep metric learning has been the focus of several research efforts~\cite{ml1,ml2,ml3}. Contrastive loss~\cite{hadsell2006dimensionality, hu2014discriminative} and triplet losses~\cite{schroff2015facenet, cheng2016person} are being widely used in several applications. In contrastive learning~\cite{khosla2020supervised, weinberger2009distance}, samples with similar labels are pulled together in the fusion embedding space while those with dissimilar labels are pulled apart. Contrastive loss has two ingredients for a given anchor: pool samples with similar labels (positive) and those with dissimilar labels (negative), and minimize distances in the embedding space for the former and maximize distances for the latter. Triplet loss uses anchors, where one positive and negative sample are chosen per anchor, which are typically the hardest examples for a given anchor. Other approaches include lifted structures~\cite{oh2016deep}, n-pair losses~\cite{sohn2016improved}, quadruplets~\cite{chen2017beyond}, angular loss~\cite{wang2017deep}, adapted triplet loss~\cite{yu2018correcting}, and multi-similarity loss functions~\cite{wang2019multi} that utilize pair-wise relations across samples in a batch. A recent work~\cite{musgrave2020metric} demonstrated that gains with metric learning are more modest compared to what is commonly reported. In this paper, we use the Multi-Similarity contrastive loss function in combination with the refiner network. We call this method ReFNet\textsubscript{MS}, wherein the responsible weights are trained in combination with maximizing separation of dissimilar embeddings in the fusion space, that can simultaneously elicit the underlying graphical structure across modalities and samples. \par

\begin{figure*}
\begin{center}
\includegraphics[width=0.92\linewidth]{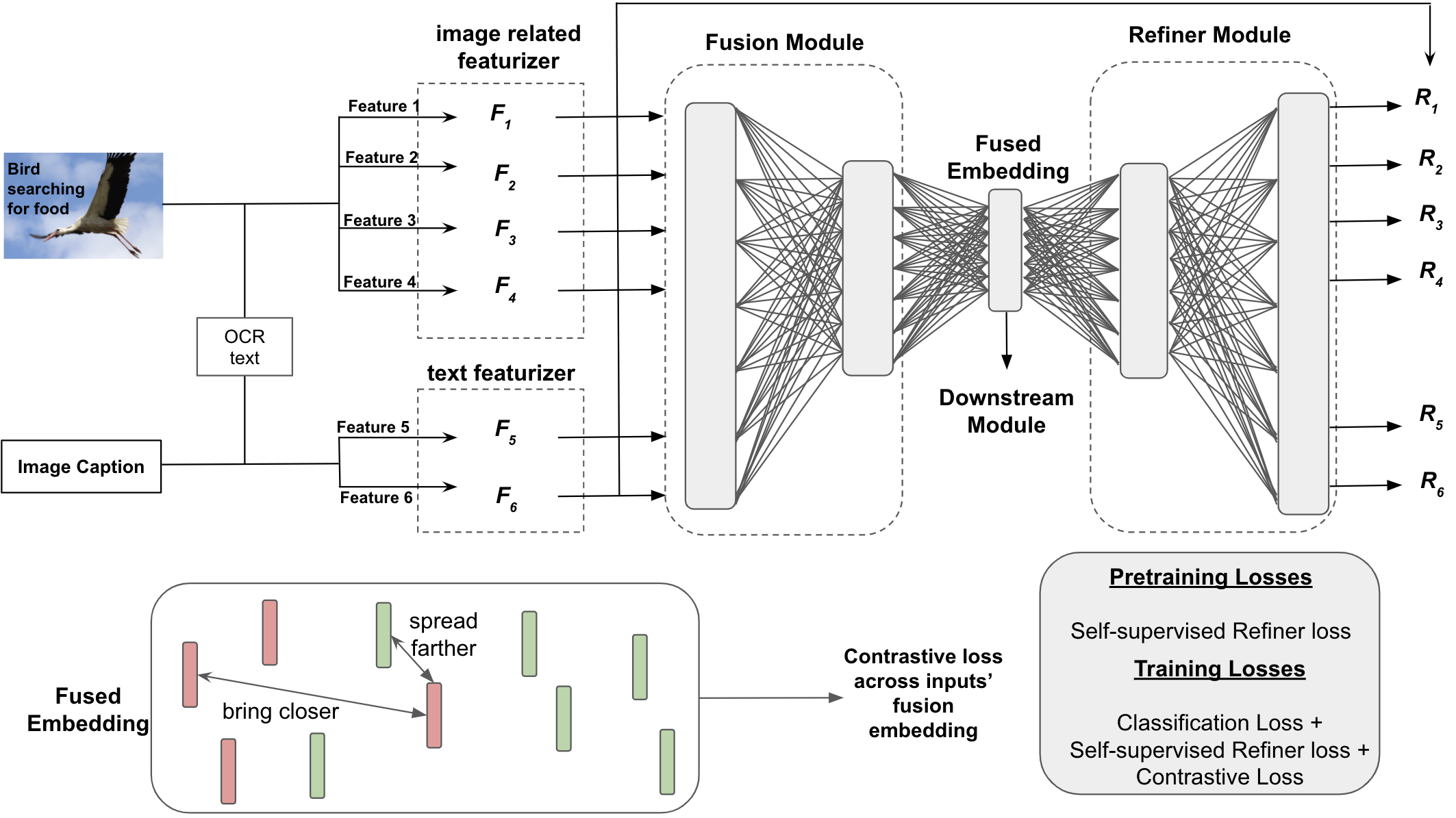}
\caption{Schematic of the proposed algorithm with a refiner and contrastive loss module. The image based features $F_1,F_2,F_3,F_4$ and text based features $F_5, F_6$ are fused together, and a refiner is applied on the fused embedding to generate refiner outputs $R_1, R_2, \cdots, R_6$ which are used to define a self-supervised loss function and a supervised Multi-similarity contrastive loss is also used across samples in a batch.}
\label{fig2}    
\end{center}
\end{figure*}
\section{Multimodal Refiner Fusion Network Design}
Let $F_1$, $F_2$, $\cdots$ $F_M$ refer to featurized inputs of $M$ modalities to a fusion module. A fusion module then aggregates these inputs, and creates a fused embedding of the multi-modal features as \par

\begin{equation}
{\text{F}}_{\text{emb}} = \mathcal{A}([F_1, F_2, \cdots F_M])
\end{equation}
where $\mathcal{A}$ maps the input features to an embedding space. In the context of this paper, $\mathcal{A}$ can encompass many of the fusion methods in literature such as Concat/MLP or other fusion modules used with Concat-Bert, ViLBERT, MMBT, etc. The fused embedding is subsequently used to train a downstream task such as classification. \par 

The refiner module that we introduce in this paper decomposes the fused embedding into set elements, and imposes a refiner loss that the decomposed set elements can capture the featurized inputs to the fusion module (Fig.~\ref{fig2}). 

\begin{equation}
     R_i = \mathcal{D}_i({\text{F}}_{\text{emb}}) \forall i = 1, 2, \cdots M 
\end{equation}
where $R_i$ are the set elements generated by the refiner module, $\mathcal{D}$. We then introduce a self-supervised loss function, $\mathcal{C}_{\text{i,ss}}$, for each fusion input
\begin{equation}
    \mathcal{C}_{\text{i,ss}} = 1 - \text{Cosine Similarity}(R_i, H_i(F_i)) \forall i = 1, 2, \cdots M
\label{selfsuper}
\end{equation}
where $H_i$ is a mapping of the features to the refiner space. In general $H_i$ can be the identity transformation unless the problem necessitates a lower dimensional refiner space than the feature space (e.g. when not sufficient training samples are present). 
The total refiner cost function is $\sum_{i=1}^{M} \gamma_i  \mathcal{C}_{\text{i,ss}}$ where $\gamma_i$ are the weights of the refiner cost function associated with the different features. \par 

\subsection{Refiner as a self-supervision module}
One advantage of the refiner module is that it is self-supervised and can leverage unlabeled data (refer Fig.~\ref{fig2}). Hence, unsupervised data, when available, can be pre-trained initially by the refiner module before training on the downstream task. This helps in reducing the amount of training labels required, especially when considering that labeling on multi-modal tasks is typically more expensive than in unimodal tasks such as classification.\par

\subsection{Refiner Fusion Module on Multi-modal transformers}
Refiner module can be applied on top any fusion architecture that takes as input feature streams from the multiple modalities that are finally fused together. These streams can be in the form of a set or a graph. Multi-modal transformers such as MMBT~\cite{kiela2019supervised}, Visual Bert~\cite{li2019visualbert} and ViLBERT~\cite{lu2019vilbert} have emerged as strong multi-modal models in the recent past. They combine the power of BERT model~\cite{devlin2018bert} for processing text, caption or OCR, with powerful ResNet models~\cite{he2016deep} to capture image features. In ViLBERT, a multi-modal co-attention model is used that has demonstrated powerful state-of-the-art performance on many public multi-modal benchmark tasks. \par 

In the rest of the paper, we apply ReFNet on top of ViLBERT architecture. The output post the co-attention layers of the text and visual streams are fused, and a decoder is applied to the fused embedding to decode back the text and visual embeddings using an MLP with hidden layers. \par 

\section{Inducing Latent Graph Structures}
While we do not explicitly model graph in this paper, we show in this section that our proposed method can exploit hidden graphical connections in the data, both across modes within a training sample and across samples in a batch as illustrated in Figure~\ref{fig3}. In Theorem~\ref{theorem:inductive_graph}, we show that when an (unknown) adjacency matrix, ${\bf A}$, exists that contains the connections across modalities and when the fusion network and refiner network are linear, then the inverse of the weights of the refiner network contains the weighted adjacency matrix. \par 



\begin{theorem}\label{theorem:inductive_graph}
Let ${\bf A}$ be an unknown adjacency matrix and ${\bf W}$ be the weights of an affine transformation to generate the fusion embedding, ${\bf E}$. The weights $\Gamma$ of a linear refiner satisfy the property, $\Gamma {\bf W} {\bf A} = {\bf I}_{m}$. \par 
\end{theorem}
\begin{proof}
Let ${\bf F}^{m \times d}$ represent features where $m$ is the number of modalities and $d$ is the size of each feature vector. We assume that each modality has a feature vector of dimension $d$ without any loss of generality (otherwise they can be padded with zeros). \par 

The unknown adjacency matrix, ${\bf A}^{m \times m}$ contains ones whenever modality $i$ and $j$ have an edge between them. Let ${\bf W}^{k \times m}$ be the unknown coefficients that create the fusion embeddings, ${\bf E}^{k \times d}$ from the features as shown in Equation~\ref{fuse_embed}.
\begin{equation}
    {\bf E} = {\bf W} {\bf A} {\bf F}
    \label{fuse_embed}
\end{equation}
where ${\bf E}$ is the embedding in a $k \times d$ space. The fusion module generates weights, ${\bf W}^{*} = {\bf W} {\bf A}$. The refiner calculates weights $\Gamma^{m \times k}$ such that $\tilde{\bf F} = \Gamma {\bf E}  = \Gamma {\bf W} {\bf A} {\bf F}$. Since the refiner network finds weights such that $\tilde{\bf F} = {\bf F}$ (refer Eq.~\ref{selfsuper} where ${\mathcal D}({\bf x}) = \Gamma {\bf x}$ and $H_i$ is identity),  we have
\begin{equation}
    {\bf F} =  \Gamma {\bf W} {\bf A} {\bf F}. \nonumber 
\end{equation}
Since the above holds $\forall {\bf F}$, $\Gamma {\bf W}{\bf A} = {\bf I}_{m}$ where ${\bf I}_m$ is an identify matrix of size $m \times m$.

If ${\bf W}{\bf A}$ is invertible, then $\Gamma = ({\bf W}{\bf A} )^{-1}$. Without using refiner, the weights, ${\bf W}^*$ will be tuned by a downstream task of much lower dimension than the refiner module, and therefore the weights will be tuned to generate a good representation of the graph in a much lower dimensional space, thereby failing to induce the latent graphical structure. 
\label{theorem_fin}
\end{proof}
\vspace{-0.15in}
\begin{lemma}
When $k=m$ and ${\bf W}$ has a rank equal to $m$, $\Gamma^{-1}$ is the weighted adjacency matrix, ${\bf W}{\bf A}$. 
\end{lemma}
\noindent The above follows directly from Equation~\ref{theorem_fin} because ${\bf W} {\bf A}$ has a rank of $m$.

\begin{figure}[]
\includegraphics[scale=0.42]{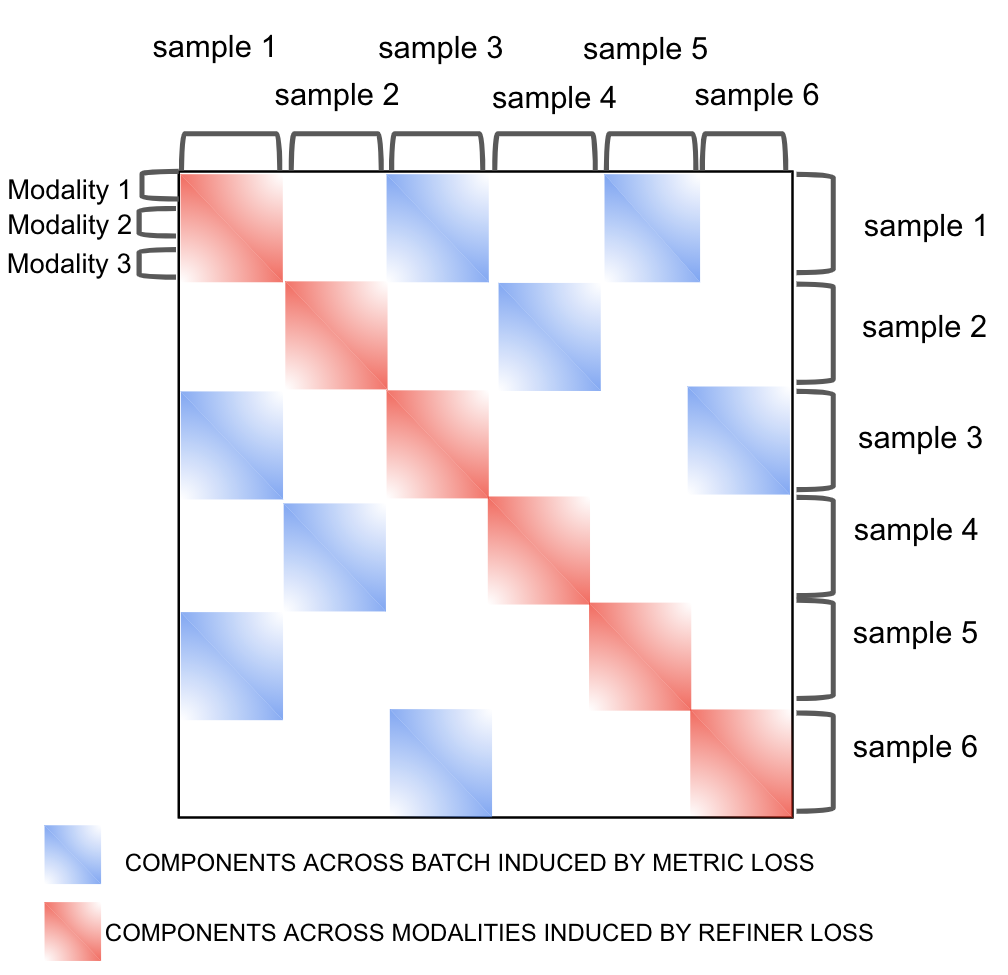}
\caption{Illustration of the connections that the refiner induces (inter-modality) and those that the metric learning induces (across samples).}
\label{fig3}
\end{figure}

\subsection{Contrastive Loss}
When the self-supervised refiner module is combined with a contrastive loss, weights are driven to induce graphical structures both within modalities and across samples. Without refiner, metric learning maximizes separation of fused embeddings based on the downstream classification task. But the addition of refiner to metric learning enables separation also across modalities (because the refiner is responsible) as we demonstrate with a T-SNE figure later in the manuscript. We use the Multi-similarity loss in this paper, though other contrastive loss functions can be used. The Multi-similarity loss for $T$ training samples in a batch is calculated as~\cite{wang2019multi}
\begin{eqnarray}
    {\mathcal L}_{MS} = \frac{1}{T} \sum_{i=1}^{T} { \left[ \frac{1}{\alpha} log[1+\sum e^{-\alpha(S_{ik}-\lambda)}]\right]} + \nonumber \\ { \left[ \frac{1}{\beta}log[1 + \sum e^{\beta(S_{ik}-\lambda)}] \right]}
\end{eqnarray}
where $S_{ij}$ is the similarity (dot product) between samples and $\alpha$, $\beta$ and $\lambda$ are hyperparameters. We provide the overall algorithm in Alg.~\ref{alg:refnet} where $\gamma_i$ are the self-supervised loss coefficients, $\zeta$ is the loss coefficient for contrastive loss, $w_k$ are the weights of the fusion and refiner networks and $\eta$ is the learning rate. \par 



\begin{algorithm}
\caption{Algorithm for training Multi-modal Fusion Networks}
\vspace{0.1in}
\begin{algorithmic}\label{alg:refnet}
\STATE \underline{\bf Pretraining (optional)}
\vspace{0.1in}
\STATE Compute features, $F_1, F_2, \cdots, F_M$.
\WHILE{${\text epoch} \le \text{max epochs}$}
\STATE ${\mathcal L}_{\text pretrain} = \sum \gamma_i {\mathcal C}_{i,ss}$
\STATE $w_{k+1} = w_{k} - \eta \frac{\partial {\mathcal L}_{\text pretrain}}{\partial w}$
\ENDWHILE \\
\vspace{0.1in}
\STATE \underline{\bf Training}
\vspace{0.1in}
\WHILE{${\text epoch} \le \text{max epochs}$ and stop criterion is not met}
\STATE ${\text{F}}_{\text{emb}} = \mathcal{A}([F_1, F_2, \cdots F_M])$
\STATE $R_i = \mathcal{D}_i({\text{F}}_{\text{emb}}) \forall i = 1, 2, \cdots M $
\STATE $\mathcal{C}_{\text{i,ss}} = 1 - \text{Cosine Similarity}(R_i, H_i({\mathcal F}_i)$
\STATE ${\mathcal L}_{\text train} = {\mathcal L}_{downstream} + \sum \gamma_i {\mathcal C}_{i,ss} + \zeta * {\mathcal L}_{MS}$
\STATE $w_{k+1} \gets w_{k} - \eta \frac{\partial {\mathcal L}_{\text train}}{\partial w}$
\ENDWHILE
\end{algorithmic}
\end{algorithm}
\section{Experiments}
For this study, we use the pre-trained ViLBERT model~\cite{lu2019vilbert} as a baseline for training and testing ReFNet and ReFNet\textsubscript{MS}. Similar to~\cite{kiela2020hateful}, the ViLBert model was only unimodally pretrained. Instead of using the element-wise product~\cite{lu2019vilbert} to fuse the unimodal representations of images and texts, we concatenate them and feed them into a new linear layer to formulate the overall representation. We are able to reproduce the baselines reported for each of examples in this section. We feed the fused embedding to two new linear layers. The Multi-similarity loss function is integrated with the fused embedding, and included during the final training runs. We used the MMF framework~\cite{singh2020mmf} built on PyTorch~\cite{paszke2019pytorch} for setting up the training and evaluation pipelines. 

\subsection{MM-IMDB}
The multi-modal IMDB dataset~\cite{arevalo2017gated} contains 25,959 movies and their poster, genre, plot and other metadata fields such as year, language, writer, director, and aspect ratio.  The goal is to classify the movie into 24 categories. Each movie can contain more than one class. The micro-f1 and macro-f1 scores were used to evaluate performance. The original dataset contains a baseline using Gated Multimodal units~\cite{arevalo2017gated}. The ViLBERT baseline~\cite{singh2020we} is used here which improved upon the Gated Multimodal methods, and is compared with ReFNet and ReFNet\textsubscript{MS}. Since each movie can simultaneously have multiple classes present, the precision and recall scores are calculated based on the f-score as follows~\cite{madjarov2012extensive}. \par

The macro f1 score is calculated from the precision, $p_j$ and recall $r_j$ of each class as 
\begin{equation}
    f_1^{\text{macro}} = \frac{1}{N} \sum_{i = 1}^{N} \frac{2 \times p_j \times r_j}{p_j + r_j}. \nonumber 
\end{equation}
The micro f1 score is calculated using all the class labels together as 
\begin{equation}
    f_1^{\text{micro}} = \frac{2 \times p^{\text{micro}} \times r^{\text{micro}}}{p^{micro} + r^{micro}} \nonumber 
\end{equation}
where $p^{micro}$ and $r^{micro}$ are the precision and recall across all classes calculated based on the total number of true positives, false positives and false negatives. \par 

We used the AdamW optimizer, with a Cosine warmup and Cosine decay learning rate scheduler. The value of $\epsilon$ for AdamW optimizer was set to $1e^{-8}$ with corresponding $\beta_1$ and $\beta_2$ as 0.9 and 0.999. The batch size was set to 32, learning rate was set to $5e^-5$ and the fused embedding dimension was set to $512$. An MLP with a hidden layer was used for the decoding refiner module. For the metric loss function, values of $\alpha=50$ and  $\beta=2$ were chosen. We used values of $\eta_1 = 0.1, \eta_2 = 0.1, \zeta = 0.1$.
\par

\subsection{Hateful Memes}
Hateful Memes dataset~\cite{kiela2020hateful} contains over 10,000 multi-modal examples (image and test) with the goal of detecting if an input is hateful or not. The dataset is constructed such that unimodal models struggle and only multi-modal models can succeed (see Fig.~\ref{fig4}). Difficult examples (“benign confounders”) are added to the dataset to make it hard to rely on unimodal signals. The dataset comprises of five different types of memes: multimodal hate, where benign confounders were found for both modalities, unimodal hate where one or both modalities were already hateful on their own, benign image and benign text confounders and finally random not-hateful examples. There were 1,000 samples in the validation dataset and 2,000 examples in the test dataset. We used values of $\eta_i = 0.1, \zeta = 0.25$.

\begin{figure}
\includegraphics[scale=0.4, clip]{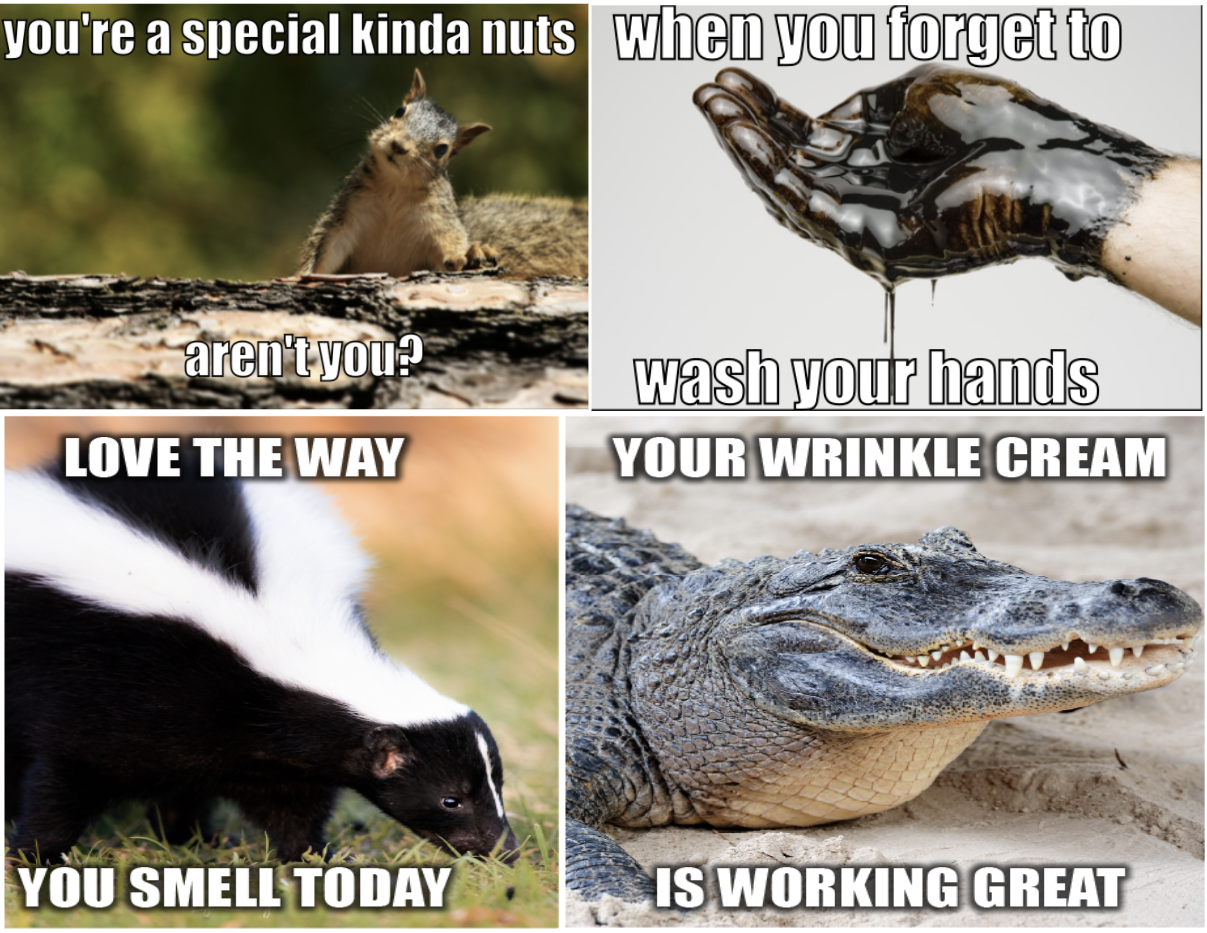}
\caption{A few examples of memes that may be considered benign (top) or mean (bottom).}
\label{fig4}
\end{figure}

We use a cross entropy loss for the two-label classification. We train on 8 Navidia Volt100 GPUs with a total batch size of 32 for a total of 22000 updates. We evaluate every 1000 updates and save the model with the best AUROC metric on the validation set. We use the AdamW optimizer with an initial learning rate of 1.0e-05. We use a linear decay learning rate schedule with warm up to train the model.

\subsection {SNLI Visual Entailment}
The SNLI Visual Entailment (SNLI-VE) dataset ~\cite{xie2019visual} consists of image-sentence pairs whereby a real world image premise and a natural language hypothesis are given. The goal is to determine if the natural language hypothesis can be concluded given the information provided by the image premise. Three labels, entailment (hypothesis is true), neutral or contradiction (hypothesis is false), are assigned to image-sentence pairs. The dataset has 550k image-sentence pairs generated based on the Stanford Natural Language Inference (SNLI)~\cite{bowman2015large} corpus and Flickr30k~\cite{plummer2016flickr30k} dataset. The training setup is the same as Hateful Memes dataset except we use a batch size of 480, an initial learning rate of 5.0e-05, and a total of 10000 updates. We used values of $\eta_i = 0.1, \zeta = 0.05$.

\section{Results}
\subsection{MM-IMDB}
Table~\ref{mmimdb_result} compares the performance of ReFNet and ReFNet\textsubscript{MS} on the test set of MM-IMDB dataset to a ViLBERT baseline model in combination with different pretrained models. In general, pretraining ViLBERT reduced the performance on MM-IMDB dataset, but ReFNet was able to improve upon the performance. A relative gain of 0.8\% in the micro f1 score and 0.82\% in the macro f1 score were observed. Compared to other pretraining modules, ReFNet had gains between 0.48\% and 2.12 \%. The improvement on micro and macro f1 score using ReFNet\textsubscript{MS} were statistically significant based on a t-test (p = 0.02 and 0.03 respectively).\par
\begin{table}
\begin{tabular}{|c|c|c|}
\hline 
model & macro f1 test. & micro f1 test. \\
\hline
\hline
\small{ViLBERT} & 58.48 $\pm$ 0.25 & 66.77 $\pm$ 0.14 \\
\small{ViLBERT-VQA2} & 57.70 & 66.42  \\
\small{ViLBERT-COCO} & 57.72 & 65.63 \\
\small{ViLBERT-cc small} & 58.20 & 66.70 \\
\small{ReFNet} & 58.75 $\pm$ 0.07  & 67.02 $\pm$ 0.15 \\
\small{ReFNet\textsubscript{MS}} & 58.96 $\pm$ 0.09 & 67.31 $\pm$ 0.19 \\
\hline

\end{tabular}
\caption{Comparison of the macro and micro f1 score on the test set across ViLBERT in combination with different pretraining datasets~\cite{singh2020we}.}
\label{mmimdb_result}
\end{table}

\subsection{Hateful Memes}
Table~\ref{hatefulmeme_result} compares the performance of ReFNet and ReFNet\textsubscript{MS} with respect to the ViLBERT baseline. ReFNet improves the accuracy by 3.30\% with a relative gain in the AUC Of 1.84\%. The use of ReFNet\textsubscript{MS} improves the Accuracy further by 0.80\%. The overall relative gain in AUC was 2.17\%. Based on a t-test, both ReFNet and ReFNet\textsubscript{MS} had a statistically significant improvement on the AUC (p-value = 1e-4 and 0.006 respectively) and the accuracy (p-value $<$ 1e-6 for both) on the test set.
\begin{table*}[!t]
\centering 
\begin{tabular}{|c|c|c|c|c|}
\hline 
model & Acc. val. & AUC val. & Acc. test. & AUC test. \\
\hline
\hline
\small{base} & 60.71 $\pm 0.29$ & 70.62 $\pm 0.42$ & 59.70 $\pm 0.20$ & 70.53 $\pm 0.07$ \\
\small{ReFNet} & 62.45 $\pm$ 1.09 & 70.87 $\pm$ 0.41 & 63.00 $\pm$ 0.31  & 71.83 $\pm$ 0.13\\
\small{ReFNet\textsubscript{MS}} & 63.29 $\pm$ 1.31 & 70.99 $\pm$ 0.37 & 63.80 $\pm$ 0.36 & 72.06 $\pm$ 0.49\\
\hline
\end{tabular}
\caption{Comparison of the Accuracy and AUC on both the validation and test sets for the Hateful Memes dataset. "base" refers to the baseline ViLBERT model.}
\label{hatefulmeme_result}
\end{table*}

\subsection{SNLI Visual Entailment}
ReFNet showed a small relative improvement on the test set (improvement in accuracy $0.1\% \pm 0.07$ on the test set which was not significant). However, ReFNet\textsubscript{MS} improves the Accuracy relatively by $0.71\%$ which was significant (p-value = 0.001). Since the dataset contains more than a hundred thousand examples, even a 1\% improvement results in thousands of images being correctly classified. 


\begin{figure}[hbt!]
\includegraphics[width=\columnwidth]{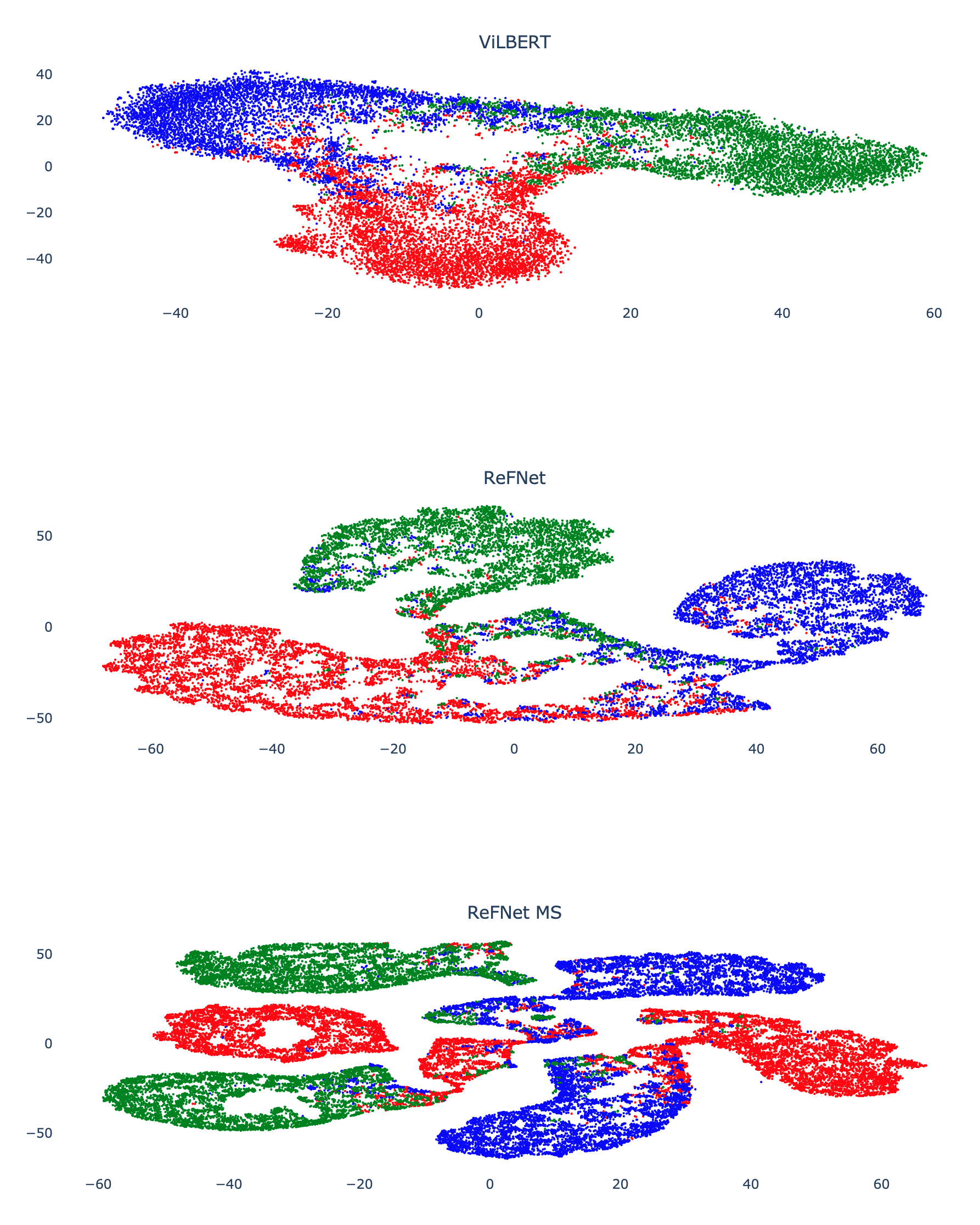}
\caption{Visualization of fusion features in reduced dimensions using T-SNE with perplexity set to 100. Top: fusion features of ViLBERT baseline showing the 3 clusters with entanglements. Middle: fusion features of ReFNet showing the 3 clusters are better separated with less entanglements. Bottom: fusion features of RefNet\textsubscript{MS} showing the Refiner and Contrastive loss inducing six clusters across modalities and classes (three clusters for each of the vision and text modalities). The colors red, blue, and green represent three classes contradiction, entailment, and neutral.}
\label{fig:fusion_figure}
\end{figure}

\section{Ablation Studies}
For the ablation study for MM-IMDB, we successively chose a fraction of the labeled dataset (5, 10 and 20\%) for downstream classification while the self-supervised refiner module leverages unlabeled data. The baselines were rerun using the ViLBERT model. Metrics were reported on the test dataset based on the model corresponding to the best validation performance on 20000 iterations. ReFNet was able to achieve a 4.03\% higher micro f1 and 11.83\% higher macro f1 score compared to the ViLBERT baseline. ReFNet\textsubscript{MS} boosted the performance over baseline to 6.51\% micro f1 score and 13.77\% macro f1 score. These were 1.66 and 5.00 when 10\% of the labeled data was used and 0.88 and 1.04 when 20\% of the labeled data was used. A full summary of the ablation study is provided in Table~\ref{mmimdb_ablation} and illustrated in Fig.~\ref{fig6}.\par 
\begin{figure*}
\includegraphics[scale=0.43]{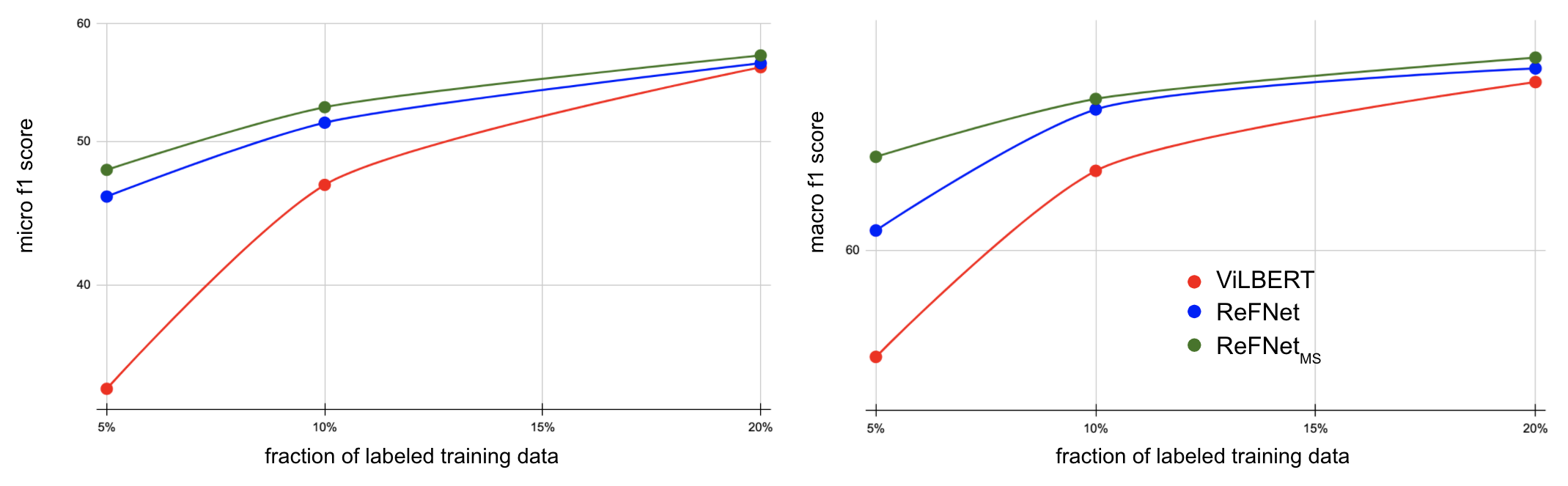}
\caption{Comparison of performance of the baseline ViLBERT model, ReFNet and ReFNet\textsubscript{MS} across datasets with 5\%, 10\% and 20\% of labeled training data available. Figure on the left shows micro and macro f1 scores on the validation dataset and that on the right shows the scores on the test dataset.}
\label{fig6}
\end{figure*}
 
\begin{table}[ht]
\begin{tabular}{|c|c|c|c|}
\hline 
model & frac labeled & macro f1 test & micro f1 test \\
\hline
\hline
\small{baseline} & 5\% & 34.06 & 56.62 \\
\small{ReFNet} & 5\% & 45.89 & 60.65 \\
\small{ReFNet\textsubscript{MS}} & 5\% & 47.83 & 63.13 \\
\hline
\small{baseline} & 10\% & 46.73 & 62.65\\
\small{ReFNet} & 10\% & 48.70 & 63.41 \\
\small{ReFNet\textsubscript{MS}} & 10\% & 51.73 & 64.31 \\
\hline
\hline
\small{baseline} & 20\% & 56.08 & 65.75 \\
\small{ReFNet} &  20\%  & 56.43 & 66.24 \\
\small{ReFNet\textsubscript{MS}} &  20\%  & 57.12 & 66.63 \\
\hline
\end{tabular}
\caption{Ablation study on the MM-IMDB dataset. Fraction labeled is the fraction of labeled samples used during training for the logit binary cross entropy function.}
\label{mmimdb_ablation}
\end{table}

On the Hateful Memes dataset, the results based on successive reduction in the fraction of labeled data are summarized in table~\ref{hatefulablation}. Using just 5\% of the labeled data, the area under the ROC curve on the test set improves by 2.52\% and 3.84\% using ReFNet and ReFNet\textsubscript{MS} respectively. The corresponding improvements in accuracy are 4.14\% and 4.60\%. When using 10\% of the labeled dataset, the accuracy improved by 1.50\% and the AUC improved by 0.50\%, and when using 20\% of the labeled dataset, these were 3.50\% and 2.59\% respectively. \par

\begin{table*}
\centering
\begin{tabular}{|c|c|c|c|c|c|c|c|c|c|c|c|c|}
\hline 
& 5\% V & 5\% V & 5\% T & 5\% T &
10\% V & 10\% V & 10\% T & 10\% T &
20\% V & 20\% V & 20\% T & 20\% T \\ \hline 
model & Acc & AUC & Acc & AUC &
Acc & AUC & Acc & AUC & 
Acc & AUC & Acc & AUC \\
\hline
\hline
\small{baseline} & 55.56 & 59.4 & 52.8 & 57.12 & 
56.05 & 60.32 & 54.45 & 60.28 &
61.08 & 50.20 & 50.2 & 61.31 \\
\small{ReFNet} & 56.20 & 58.65 & 56.94 & 59.64 & 
56.15 & 61.84 & 54.65 & 60.87 &
54.96 & 61.95 & 53.70 & 63.90  \\
\small{ReFNet\textsubscript{MS}} & 57.34 & 62.02 & 57.40 & 60.96 & 
55.65 & 61.82 & 54.25 & 60.78 &
58.13 & 62.78 & 58.30 & 62.86 \\
\hline
\end{tabular}
\caption{Ablation study on the Hateful Meme dataset. V refers to validation set, T refers to the test set, and 5\%, 10\% and 20\% refer to the corresponding fraction of labeled training data used. AUC is the area under the ROC curve and ACC is the accuracy of the predicted model. All the reported values correspond to the model with the best performance on the validation dataset.}
\label{hatefulablation}
\end{table*}

\section{Discussion}

We started with the hypothesis that imposing a responsibility condition on Multimodal fusion architectures can improve performance on downstream tasks. We imposed the responsibility condition using a refiner module to decode each feature feeding into the fusion. This self supervised loss function for each input feature used the cosine similarity between the decoded and original fusion features. We demonstrated this by integrating refiner on a strong transformer based multimodal baseline - the ViLBERT model. We showed that imposing responsibility can result in relative gains of over 5.53\% on the accuracy and 1.84\% on the Hateful memes dataset. \par 

Reducing the amount of labeled data available showed that the self-supervised ReFNet can effectively leverage unlabeled data and maintain performance much better than the baseline ViLBERT model. For instance, with 5\% of labeled data available for MM-IMDB, ReFNet had an improved micro f1 performance of 4.05, and adding metric loss (ReFNet\textsubscript{MS}) improved it further by 2.48, yielding a net gain score of 6.53. The macro-f1 score had a net gain of 13.77. These demonstrate how well the responsibility module combined with the metric loss function could perform compared to baseline transformer models for multi-modal classification. This is also beneficial for the purposes of multi-modal labeling of datasets. In general, labeling for multi-modal problems such as misinformation can be significantly more expensive compared to classification tasks. Leveraging ReFNet with a baseline model such as ViLBERT can help reduce the amount of labeling required to attain the same performance. \par 

Choosing the refiner space to be the actual input (e.g. image pixels) rather than the input to the fusion module resulted in a reduction in performance and was not considered to be used with the ViLBERT architecture. This is likely because the co-attention mechanism that ViLBERT models is nullified by imposing responsibility prior to applying the co-attention. However, this might provide a richer option for other architectures, especially when significantly more supervised data is available (since the input space is typically of a larger dimension than the input to the fusion module). \par 

Metric Learning methods have shown strong improvements for multi-class classification problems. When integrated with the self-supervised Refiner network, the metric Refiner network is able to elicit representations of the fused embeddings in a responsible setting, (i.e.) derive distance metrics in the embedding space characterized by strong unimodal and mixed representations of the input features. In order to demonstrate how the refiner network enables inducing graphical representations of the fused embedding in the latent space, we generated T-SNE plots using the fusion features with  ViLBERT, ReFNet and ReFNet\textsubscript{MS} algorithms. Fig.~\ref{fig:fusion_figure} shows that ViLBERT model generates 3 vaguely separated fusion feature clusters due to their entanglements while the clusters generated by ReFNet has a better separation between clusters and clearly delineates them. The metric learning algorithm is able to further separate the clusters across modalities, therefore we observe six clusters, three different classes for each modality (vision and language). This responsibility property can mitigate the problem of vanishing gradient by driving the weights to have some unimodal adherence. In addition, we can obtain performance boost when a few modalities dominate performance of the downstream task, and the fusion module is unable to encode the other modalities without the existence of a high dimensional regularizer such as the Refiner module. \par 
{\bf Limitations and Future Work:} Based on some initial experiments, we chose the Multi-Similarity loss function but did not really explore the full space of Metric Loss functions. Further, we did not showcase examples containing multi-modal problems with strong transductive graph baselines.   \par 
{\small
\bibliographystyle{ieee_fullname}
\bibliographystyle{unsrt}

\bibliography{egbib}
}
\end{document}